\documentclass[12pt,draftclsnofoot,onecolumn]{IEEEtran}

\usepackage{amsmath,amssymb}
\usepackage{pdflscape}

\usepackage{graphicx,cite,bm,amsthm,enumerate,cases,mathtools,accents}
\renewcommand{\(}{\left(}
\renewcommand{\)}{\right)}
\renewcommand{\[}{\left[}
\renewcommand{\]}{\right]}

\renewcommand{\S}{\mathbf{S}}

\newcommand{\y}{\mathbf{y}}

\newcommand{\E}{\mathbf{E}}

\newcommand{\x}{\mathbf{x}}

\newcommand{\I}{\mathbf{I}}
\newcommand{\J}{\mathbf{J}}
\newcommand{\C}{\mathbf{C}}

\newcommand{\A}{\mathbf{A}}

\newcommand{\M}{\mathbf{M}}

\newcommand{\B}{\mathbf{B}}

\newcommand{\Tr}[1]{{\rm{Tr}}\left(#1\right)}

\newcommand{\End}[1]{{\rm{End}}}

\renewcommand{\log}[1]{{\rm{log}}#1}

\renewcommand{\det}[1]{\left|#1\right|}

\newcommand{\conv}[1]{{\rm{conv}}\(#1\)}

\newtheorem{lemma}{Lemma}
\newtheorem{definition}{Definition}

\newtheorem{prop}{Proposition}
\newtheorem{corollary}{Corollary}

\newtheorem{rem}{Remark}

\newcommand{\norm}[1]{\left\lVert#1\right\rVert}

\usepackage[ruled,vlined,linesnumbered]{algorithm2e}
\usepackage{algorithmic,subfigure,enumitem}
\usepackage{fontenc,amsfonts}
\usepackage{inputenc,relsize}
\usepackage[square,sort,compress,comma,numbers]{natbib}
\usepackage{epstopdf,pst-node}

\usepackage{pgfplots,authblk}
\usepackage{verbatim}

\usepackage{pgf}

\newcommand{\cm}{\textcolor{blue}}

\begin{document}
\title{Stationary Geometric Graphical \\ Model Selection}

\author[1]{Ilya Soloveychik\thanks{This work was supported by the Fulbright Foundation and Office of Navy Research grant N00014-17-1-2075.}}
\author[2]{Vahid Tarokh}
\affil[1]{\normalsize John A. Paulson School of Engineering and Applied Sciences, Harvard University}
\affil[2]{\normalsize Department of Electrical and Computer Engineering, Duke University}
\maketitle

\begin{abstract}
We consider the problem of model selection in Gaussian Markov fields in the sample deficient scenario. In many practically important cases, the underlying networks are embedded into Euclidean spaces. Using the natural geometric structure, we introduce the notion of spatially stationary distributions over geometric graphs. This directly generalizes the notion of stationary time series to the multidimensional setting lacking time axis. We show that the idea of spatial stationarity leads to a dramatic decrease in the sample complexity of the model selection compared to abstract graphs with the same level of sparsity. For geometric graphs on randomly spread vertices and edges of bounded length, we develop tight information-theoretic bounds on sample complexity and show that a finite number of independent samples is sufficient for a consistent recovery. Finally, we develop an efficient technique capable of reliably and consistently reconstructing graphs with a bounded number of measurements. 
\end{abstract}



\section{Introduction}

\subsection{Learning Markov Random Fields}
Markov random fields, or undirected probabilistic graphical models, provide a structured representation of the joint distributions of families of random variables. A Markov random field is an association of a set of random variables with the vertices of a graph, where the missing edges describe conditional independence properties among the variables \cite{lauritzen1996graphical}. It was shown by Hammersley and Clifford in their unpublished work \cite{lauritzen1996graphical} that the joint probability distribution specified by such a model factorizes according to the underlying graph. The practical importance of Markov random fields is hard to overestimate. They have been applied to a large number of areas including bioinformatics, social science, control theory, civil engineering, political science, epidemiology, image processing, marketing analysis, and many others. For instance, a graphical model may be used to represent friendships between people in a social network \cite{farasat2015probabilistic} or links between organisms with the propensity to spread an infectious disease \cite{knorr1998modelling}. The increased availability of large-scale network data created every day by traditional and social media, sensors, mobile devices, and social infrastructure provides rich opportunities and unique challenges for the analysis, prediction, and summarization, making the development of novel large network analysis techniques absolutely necessary.

Given the graph structure, the most common computational tasks include calculating marginals, partition function, maximum a posteriori assignments, sampling from the distribution, and other questions of statistical inference. On the other hand, in many applications estimating the unknown edge structure of the underlying graph, also known as \textit{model selection} or \textit{inverse problem}, has attracted a great deal of attention. Naturally, both problems are challenging especially in high dimensional scenarios and are known to be NP-hard for general models \cite{karger2001learning,bogdanov2008complexity}.

In model selection, the naive approach of searching exhaustively over the space of all graphs is computationally intractable, since there are as many as $2^{{p \choose 2}}$ distinct graphs over $p$ vertices. Therefore, prior knowledge on the graph structure is used to make the size of the family of models tractable. A variety of methods have been proposed to address this problem. One of the first works in this direction was performed by Chow and Liu \cite{chow1968approximating}, who showed that if the underlying graph is known to be a tree, the model selection problem reduces to a maximum-weight spanning tree problem. Other models considered in the literature include sparse networks with bounded degrees of the vertices \cite{bresler2015efficiently, santhanam2012information}, walk-summable and locally separable graphs \cite{anandkumar2012high}, thresholding methods \cite{bresler2008reconstruction}, $\ell_1$-based relaxations \cite{meinshausen2006high, ravikumar2011high, friedman2008sparse, yuan2007model}, methods based on penalized pseudo-likelihood \cite{ji1996consistent}, and many others.


Most of the papers listed in the previous paragraph consider the graph selection problem in the high-dimensional setting, meaning that the number of samples $n$ is comparable to or even less than the dimension $p$ of the parameter space. One of the main focuses of these works consists in deriving tight information-theoretic lower bounds on the sample complexity, or in other words the necessary number of independent snapshots of a network that would allow reliable recovery of its connectivity structure. The benchmark results \cite{santhanam2012information, anandkumar2012high} for the Ising and Gaussian models on $d$-regular graphs claim that as a function of the model parameter $p$, the number of measurements
\begin{equation}
\label{eq:n_prop_log}
n = c\, \log\, p
\end{equation}
is required to make the learning possible. However, in many real wold scenarios even this moderate dependence may be prohibitively demanding. One of the reasons for that is the large value of the constant of proportionality $c$ between $n$ and $\log\, p$ in (\ref{eq:n_prop_log}). Quite often, this constant does not receive much attention during the analysis, while in practice its value becomes critical. Moreover, depending on the parameters of the graph for quite a wide range of dimensions $p$ the constant $c$ may be so large that the required number of samples will be essentially greater than the dimension. The aforementioned issues call for the development of new techniques that will allow to decrease the sample complexity beyond the logarithmic scaling.

In the sample deficient regime, all the parameters of the problem cannot be estimated reliably. Therefore, the cardinality of the family of admissible models should be reduced. This can be achieved by reconsidering the goal of the learning process or by introducing some additional structure into the model. Such structure can arise from the physical properties of the system, its spatial design, or can be reduced to an off-line precomputation of parameters. In this work, we rely on the spatial structure of the real networks. Physical networks are naturally embedded into Euclidean spaces. The induced spatial structure dictates the sparsity pattern of the underlying graph and suggests that the subgraphs of the entire network that have similar geometry should also exhibit similar statistical behavior. 

\subsection{Our Contribution}
In this article, we utilize geometric properties of the graphs to facilitate their structure learning through the notion of \textit{spatially stationary distributions}. Essentially, this concept is a direct generalization of stationary time series from one-dimensional chains to more general graphs. The absence of time axis makes such generalization highly non-trivial and requires proper definition and exploitation of the network topology. For concreteness and to better illustrate the idea behind the spatially stationary distributions over graphs, we focus on graphs whose vertices are uniformly randomly distributed over two dimensional regions and whose edges are not too long. We consider Gaussian graphical models over graphs of bounded vertex degree and first demonstrate that in this setup consistent learning of the entire network structure is possible even with a finite number of samples independent of the size of the graph. In addition, we develop an efficient adaptive method capable of consistently learning all the edges of the network with the number of measurements independent of $p$.

It is important to emphasize that our notion of \textit{spatial stationarity} is completely different from the concept of stationary graphs utilized in the graph signal processing literature \cite{perraudin2017stationary, marques2016stationary}. In these papers, the framework is based on the harmonic analysis over graphs and is not tied to their geometry, while we directly extend the notion of stationarity from time series using the spatial structure of the underlying time counts.

The rest of the paper is organized as follows. We set up notation, provide a number of motivating examples, and introduce the concept of spacial stationarity in Section \ref{sec:stat_gr}. To make the technical derivations more transparent and emphasize the main ideas, in Section \ref{sec:gauss_stat} we narrow down the scope of models to Gaussian distributions and make a number of standard technical assumptions. Section \ref{sec:pr_form} is devoted to the problem formulation and the lower information-theoretic sample complexity bounds. In Section \ref{sec:approach} we describe the approach underlying the proposed algorithm and in Section \ref{sec:alg} we provide an efficient implementation of the latter and discuss its performance characteristics. 
Our conclusions are given in Section \ref{sec:concl}. Technical details of the proofs can be found in Appendices \ref{app:graph_detect_inf_thoer_bound}-\ref{app:gauss_model}.


\section{Stationary Graphical Models}
\label{sec:stat_gr}
We start this section by introducing necessary notation, then we motivate our study through a number of examples, and eventually demonstrate how our framework generalizes them.

\subsection{Graphical Models}
Let $\mathcal{G} = (G,\E)$ be an undirected graph with the vertex set $G = \{1,\dots,p\}$ and binary adjacency matrix $\E$ with zero diagonal. For a vertex $i \in G$, we denote by $\mathcal{N}(i) \subset G$ the set of its neighbors. To each vertex $i$ we associate a real random variable $x_i$ and denote the probability density function of the joint distribution of $\x = \(x_1,\dots,x_p\)^\top$ by $f(\x)$. We say that $f(\x)$ satisfies the local Markov property with respect to (w.r.t.) graph $\mathcal{G}$ if
\begin{equation}
f(x_i|\x_{\mathcal{N}(i)}) = f(x_i|\x_{G\backslash \{i\}}),\quad \forall i \in G,
\end{equation}
where $\x_A = \{x_i|i\in A \subset G\}$ and $f(\x_A|\x_B)$ is the conditional density of $\x_A$ given the values of $\x_B$. More generally, we say that $\x$ satisfies the global Markov property if for all disjoint sets $A,B \subset G$, we have
\begin{equation}
f(\x_A, \x_B|\x_S) = f(\x_A|\x_S) f(\x_B|\x_S),
\end{equation}
where $S$ is a \textit{separator set} between $A$ and $B$, meaning that the removal of vertices in $S$ partitions $\mathcal{G}$ in such a way that $A$ and $B$ belong to distinct components of the remaining graph.

In our setup without loss of generality $f(\x)$ will a positive function in which case the global Markov property is equivalent to the local one \cite{koller2009probabilistic}. Given a graph $\mathcal{G}$ with the associated random vector $\x$ whose density function satisfies either of the properties, we shall say that we are given a graphical model (random Markov field).

\subsection{Motivation}
\label{sec:examples}
To better motivate our study, let us consider a number of models ubiquitous in various areas of modern statistics, computer science, and physics. We will see that all these examples are particular cases of a more general phenomena that would lead us to a natural extension.

\begin{itemize}
\item \noindent \textbf{Stationary Processes.} Assume we are given a time series, which is a sequence of random variables labeled by the integer time ticks, $\{\x_i,\; i \in \mathbb{Z}\}$. This sequence is said to be stationary \cite{priestley1981spectral} if the joint distribution of any finite number of variables is invariant under the shifts of time axis,
\begin{equation}
F(\x_{i_1},\dots,\x_{i_k}) = F(\x_{i_1+l},\dots,\x_{i_k+l}),\;\; i_1,\dots,i_k,l \in \mathbb{Z},
\end{equation}
where $F$ is the cumulative density function. It is easy to observe that the main feature of the setup making this definition possible is the fact that due to the uniform distribution of the ticks on the time axis, shifts are isomorphisms of the chain of time ticks. If on the contrary the distribution of time ticks would be non-uniform, such straightforward definition will fail.

\item \noindent \textbf{Ising Models.} Straight line with integer ticks on it is a one-dimensional special case of a $p$-dimensional integer lattice which is the Cartesian product of $p$ copies of $\mathbb{Z}$. Similarly to the previous example, let us associate with every vertex of the obtained lattice a random variable. For example, a $3$-dimensional lattice with variables assuming binary values once endowed with an appropriate joint distribution \cite{stanley1971phase, baxter2016exactly} can model the behavior of electronic spins in metals. This is an example of an Ising model (see \cite{stanley1971phase, baxter2016exactly, mccoy2014two} for the technical details). In most applications the resulting finite-dimensional distributions of variables are invariant under isomoprhisms of the lattice, which are essentially compositions of shifts in the directions of coordinate axes, rotations and reflections. Such natural invariance of the finite-dimensional marginal distributions can be thought of as a generalized stationarity in the absence of time axis.
\end{itemize}

The main observation we would like to emphasize in both examples is that the statistical properties of the distributions at hand are determined by the geometric properties of the graphs the variables are associated with. Note that in both cases the graphs are not abstract but metric, meaning that the lengths of the edges are important (and specifically in both our examples are all equal to 1). This is a fundamental property that we would like to utilize and extend in our work. Indeed, quite often in practice the underlying graphs do not possess global isomorphisms (in both stationary processes and Ising models, the transformations discussed above map bijectively the entire graph into itself), however, quite often they allow local isomorphisms. In other words, certain regions of the graphs can be similar from the geometric perspective to some other regions of the same graphs and this structural similarity will lead to similarity in relations between the corresponding variables. Below we provide the details.

\subsection{Geometric Graphs and Stationarity}
The origin and necessity of the geometric structure can also be explained from a different perspective. As mentioned in the Introduction, due to the enormous size of the set of graphs on $p$ vertices, structural constraints are usually imposed to make the learning feasible, especially in the high-dimensional regime where the number of samples is not enough to consistently estimate all degrees of freedom. Probably the earliest paper taking advantage of this approach was the seminal work of Chow and Liu \cite{chow1968approximating}, where the authors established that the structure estimation in tree models reduces to the maximum weight spanning tree problem. For graphs with loops the problem is much more challenging for two reasons: 1) a vertex and its neighbor can be marginally independent due to indirect path effects, and 2) this difficulty is amplified by the presence of long-range correlations meaning that distant vertices can be more correlated than the close ones. So far, there have not been proposed a complete description of graphs for which structure estimation is achievable, however, a number of methods allowing model selection in graphs with structure richer than trees have been suggested. Among them are such models as polytrees \cite{dasgupta1999learning}, hypertrees \cite{srebro2001maximum}, graphs with few short cycles, \cite{anandkumar2012high}, general sparse Ising models \cite{bresler2015efficiently}, graphs with large girth and bounded degree \cite{netrapalli2010greedy}, and many others. 

When a network is embedded into the physical space, its connectivity properties are determined by the ambient Euclidean geometry. More specifically, the underlying graph and coupling parameters are influenced by the specific way the network is deployed. In this work, we consider two types of structure induced by the spatial arrangement of the interacting agents.
\begin{itemize}
\item \noindent \textbf{Local.} It is natural to assume that only those vertices that are situated close to each other can be connected. This local connectivity implies sparsity of the network.
\item \noindent \textbf{Global.} Consider the behavior of two small subgraphs of the network far apart from each other and having similar geometric patterns (similar relative locations of the vertices). In certain cases, it is reasonable to assume that the marginal distributions over these two subgraphs should be close. In other words subgraphs with similarly arranged vertices should also have similar distributions. This property will allow us to estimate the parameters of each of the subnetworks more precisely since we can utilize the samples from one of them for the structure learning of the other. 
\end{itemize}
Next we formalize these intuitive properties. To set up notation, assume that the ambient space is $\mathbb{R}^2$ with a fixed orthonormal basis and standard scalar product\footnote{The multidimensional extension is straightforward.} and $G = \{\y_1,\dots,\y_p\}\subset \mathbb{R}^2$. For any two subsets $F = \{\y_1^F,\dots,\y_r^F\}$ and $H = \{\y_1^H,\dots,\y_r^H\}$ of $G$ of the same cardinality $r$, denote the \textit{Euclidean distance}\footnote{Not to be confused with graph distance between vertices also utilized below.} between them by
\begin{equation}
\label{eq:perm}
\tau\(F,H\) = \min\limits_{\pi \in S_r} \max_{i=1}^r \norm{\y_i^F - \y_{\pi(i)}^H},
\end{equation}
where $S_r$ is the group of permutations on $r$ elements and $\norm{\x}$ is the Euclidean norm of $\x$.
\begin{definition}
\label{def:rigid_trans_def}
Given two sets $F,\; H \subset G$ of the same cardinality on the plane, the similarity measure between them is defined as
\begin{equation}
\rho(F,H) = \min_{R}\tau\(F,R(H)\),
\end{equation}
where $R$ runs through all rigid transformations of the plane (compositions of translations and rotations).
\end{definition}

\begin{definition}
Given two sets $F,\; H \subset G$ of the same cardinality on the plane, we say that $H$ is an $\varepsilon$-copy of $F$ if
\begin{equation}
\rho(F,H) \leqslant \varepsilon.
\end{equation}
\end{definition}


For simplicity of notation, we identify the points of $G$ with their indices. Consider an undirected connected graph $\mathcal{G}$ over the points in $G$. As above, associate with every vertex $i \in G$ a random variable $x_i$ and consider a graphical model over $\mathcal{G}$ with the density function $f(\x)$. 

Our next goal is to define the notion of \textit{spatial stationarity} for a distribution over $\mathcal{G}$. Consider the set of probability measures over $\mathbb{R}^p$ and let $d(\cdot,\cdot)$ be a metric over this set. Later, we make a specific choice of such metric suitable for the Gaussian setup that we will treat.


\begin{definition}
\label{def:stat_prob}
We say that a graphical model over $\mathcal{G}$ with the density function $f(\x)$ is spatially stationary w.r.t.\ metric $d(\cdot,\cdot)$ if there exists a constant $\gamma$ such that for any two subgraphs $\mathcal{F} = (F,\E_F)$ and $\mathcal{H} = (F,\E_F)$ of $\mathcal{G}$ of the same order,
\begin{equation}
\label{eq:def_stat_gr}
d\(f(\x_{\mathcal{F}}),f(\x_{\mathcal{H}})\) \leqslant \gamma \rho(F,H).
\end{equation}
\end{definition}

This notion of spatial stationarity with regard to a geometric graph is a natural generalization of both stationary processes and Ising models. Indeed, in both examples discussed in section \ref{sec:examples} the shifts, rotations and/or reflections are graph isomorphisms, meaning that the graphs are bijectively mapped onto themselves. In such case, according to (\ref{eq:def_stat_gr}) the right-hand side is zero and forces the left-hand side to be zero as well and this is indeed the case in both stationary time series and Ising models. In words, the exact match of the geometric patterns of the (sub-)graphs imply the exact matching of the joint distributions over them. Definition \ref{def:stat_prob} extends this idea to more general geometric graphs. Here we allow small deviations of the pattern and simultaneously loosen the requirement on the distributions of the translated versions of small subgraphs from being strictly invariant to approximately invariant.

\section{Gaussian Stationary Graphical Models}
\label{sec:gauss_stat}
\subsection{Gaussian Setting}
In this work for simplicity we focus on Gaussian graphical models meaning that the joint distribution $f(\x)$ reads as
\begin{equation}
f(\x;\J) = \frac{1}{\sqrt{2\pi |\J^{-1}|}}e^{-\frac{1}{2} \x^\top \J \x},
\end{equation}
where $\J =\{\J_{ij}\}_{i,j=1}^p = \bm\Theta^{-1}$ is the precision (inverse covariance, potential, information) matrix parameter of the population and the distribution is assumed to be centered. It can be easily shown \cite{lauritzen1996graphical} that $\J$ has zeros in the entries $(i,j)$ corresponding to the missing edges in $\E$ and is non-zero otherwise. Both $\J$ and $\bm\Theta$ are assumed to be positive definite, making the distribution non-degenerate. The off-diagonal non-zero elements $\J_{ij}$ are referred to as \textit{coupling} or \textit{edge} parameters between vertices $i$ and $j$ and in this work are assumed to be positive,
\begin{equation}
\forall\; i \neq j \in V \colon \quad \J_{ij} \neq 0 \;\;\iff\;\; \E_{ij} = 1\;\; \text{and}\;\; \J_{ij} > 0.
\end{equation}
Remarkably, the case of equal coupling parameters already captures the essential complexity of the problem as long their value satisfies certain correlation decay conditions (see Section \ref{sec:cdp_prop} for details). Therefore to avoid bulky technical details, we assume
\begin{equation}
\forall\; i \neq j \in V \colon \quad \J_{ij} = \theta\E_{ij}.
\end{equation}
The general case of model selection with different edge parameters is treated similarly. It is also common in the graphical model selection literature to assume the conditional variances of all the variables $\J_{ii}$ to be equal to one \cite{anandkumar2012high}. Thus, we can write
\begin{equation}
\J = \I + \theta\E,
\end{equation}
for the precision matrix of the population.

\subsection{Technical Assumptions}
\label{sec:cdp_prop}
Below we make a number of standard technical assumptions \cite{anandkumar2012high, bresler2015efficiently, santhanam2012information}.
\begin{enumerate}[label=\textbf{[A\arabic*]}]
\item \label{assm:1} \noindent \textbf{Bounded Degree.} It is a common practice to restrict the family of admissible models to graphs with bounded vertex degree to enforce certain level of sparsity \cite{anandkumar2012high, bresler2015efficiently, santhanam2012information}. In our setup, it is natural to assume the graphs to have bounded degree due to locality of interactions and geometric properties of our networks. Formally, we assume that the degrees of the vertices are bounded by $d$,
\begin{equation}
|\mathcal{N}(i)| \leqslant d,\quad \forall\; i \in G.
\end{equation}
and the lengths of edges between connected vertices are bounded by a constant,
\begin{equation}
\label{eq:beta_bound}
j \in \mathcal{N}(i) \quad \Rightarrow \quad \norm{i-j} \leqslant \beta.
\end{equation}
Below we focus more on the allowed values of $\beta$.
 
\item \label{assm:2} \noindent \textbf{Spectral Properties.}
Except for the sparsity, successful structure estimation also relies on certain assumptions on the parameters of the model, and these are often tied to specific algorithms. Among various assumptions of this type, the correlation decay property (CDP) stands out. Informally, a graphical model is said to have the CDP if any two variables $x_i$ and $x_j$ are asymptotically independent as the graph distance between $i$ and $j$ increases. Most of the existing model selection procedures require CDP explicitly \cite{montanari2009graphical}, the rest often do so indirectly through different assumptions on the model parameters and likely also require the CDP (see the survey \cite{gamarnik2013correlation} as well as e.g. \cite{dobrushin1970prescribing}). For example, the authors of \cite{anandkumar2012high} require a Gaussian Markov field to be $\alpha$-walk summable, as introduced and analyzed in \cite{malioutov2006walk}. The property of $\alpha$-walk summability essentially means that the spectral norm of the matrix consisting of the element-wise absolute values of the partial correlations is bounded by $\alpha$. Roughly speaking, this condition guaranties invertibility of the precision matrix (or in other words existence and non-degeneracy of the covariance matrix of the population). As can be traced from \cite{anandkumar2012high}, for example in ferromagnetic models with the vertex degrees tightly concentrated around a fixed value, the $\alpha$-walk summability is almost equivalent to an upper bound on the coupling parameters and is, therefore, a close relative of the CDP. Another example is the algorithm introduced by \cite{ravikumar2010high}, which is shown to work under certain incoherence conditions that seem distinct from the CDP, however, \cite{montanari2009graphical} established through a careful analysis that the algorithm fails for simple families of certain Markov fields (ferromagnetic Ising models) without the CDP. In general, some assumptions that involve incoherence conditions are often hard to interpret as well as verify \cite{meinshausen2006high, ravikumar2011high}. It is also worth mentioning that usually, in addition to upper limits, the correlations between neighboring variables are supposed to be bounded away from zero (as is true for the ferromagnetic Ising model in the high temperature regime) to make the family of models identifiable.

\medskip
In this work, we assume that the coupling parameter satisfies
\begin{equation}
\label{eq:spec_assm}
d\theta < \frac{1}{2}.
\end{equation}
The equivalence of this condition to the CDP is justified by Lemma \ref{lem:cdp} below. 

\item \label{assm:3} \noindent \textbf{Vertex Distribution.}
To fix a concrete setup, we assume that the elements of $G$ are uniformly distributed over a square $S$ with side length $s$ and that it satisfies the relation
\begin{equation}
\label{eq:density}
p = \eta s^2,
\end{equation}
where $\eta$ is the density constant. It is worth mentioning, that the framework developed below will suite different other settings and we only use the uniform model for the vertices for simplicity of exposition. We also assume that the edge length bound $\beta$ introduced in (\ref{eq:beta_bound}) is large enough, so that
\begin{equation}
\label{eq:beta_assm}
\eta \beta^2 \gg d > 1,
\end{equation}
or in other words the average number of vertices in a $\beta \times \beta$ square is large. Denote the obtained family of graphs by $\mathfrak{H}$. 
\end{enumerate}

\begin{rem}
The construction described above resembles Random Geometric Graphs (RGGs) \cite{penrose2003random}. However unlike RGGs, we do not require all the vertices inside balls of a specific radius to be connected. In addition, to the best of our knowledge the works on RGGs are usually concerned with the combinatorial properties of the latter such as the sizes of the connected components, percolation effects and similar, and do not focus on probabilistic graphical models over such graphs.
\end{rem}

\section{Problem Formulation and Lower Bounds}
\label{sec:pr_form}
Given $n$ independent and identically distributed (i.i.d.) snapshots $\mathcal{X}^n = \{\x_1,\dots,\x_n\} \subset \mathbb{R}^p$ drawn from the product probability density function $\prod_{i=1}^n f(\x_i|\mathcal{G})$ with the underlying graph $\mathcal{G} \in \mathfrak{H}$, our goal is to detect all the edges of $\mathcal{G}$. To assess the quality of the algorithm suggested below and to be able to compare it to the other existing methods, we need to specify a measure of performance. Here we follow the common trend \cite{santhanam2012information, bresler2015efficiently, anandkumar2012high} and consider the so-called zero-one loss for edges of the detected graph, meaning that we declare a failure once the estimated network differs from the ground truth graph by at least one edge. 

Next, we develop information-theoretic necessary bounds on the sample complexity of the model selection task formulated above. In other words, such bounds can be interpreted as the minimal possible number $n$ of i.i.d.\ snapshots (as a function of other parameters of the problem) such that with less samples an error-less reconstruction is impossible. In positive terms it can be formulated as the minimal possible number of samples that makes the detection with vanishing error theoretically achievable.

Assume a learning algorithm is chosen and its output is a graph $\widehat{\mathcal{G}} = \widehat{\mathcal{G}}\(\mathcal{X}^n\)$, then the probability of error is
\begin{equation}
\label{eq:prob_eq}
\mathbb{P}_e = \mathbb{P}\[\widehat{\mathcal{G}} \neq \mathcal{G}\],
\end{equation}
where the probability measure in (\ref{eq:prob_eq}) is taken w.r.t.\ the measurements sampled from the graphical model with the underlying graph $\mathcal{G}$ and to the uniform choice of the graph $G$ from the family of admissible models. It is common to approach the lower bounds \cite{santhanam2012information, anandkumar2012high} from the information-theoretic perspective as the source coding (or compression) problem \cite{cover2012elements}. If we recast our setup as reconstruction of the source given the measurements $\mathcal{X}^n$, then the necessary conditions on the sample complexity follow from Fano's inequality (see Lemma \ref{lem:fano_ineq_lemma} in Appendix \ref{app:graph_detect_inf_thoer_bound}).

\begin{prop}[Necessary Sample Complexity] 
\label{th:nec_cond_complete_graph}
Let the assumptions \ref{assm:1}-\ref{assm:3} hold and assume that a graph $\mathcal{G}$ is chosen uniformly from $\mathfrak{H}$. Then the number $n$ of i.i.d.\ samples from $f(\cdot)$ necessary for $\mathbb{P}_e$ to vanish asymptotically is
\begin{equation}
n \geqslant \frac{\log\(\frac{\eta \beta^2}{d}\)}{2\(\frac{\theta}{1-d\theta}\)^2},\quad p \to \infty.
\end{equation}
\end{prop}
\begin{proof}
The proof can be found in Appendix \ref{app:graph_detect_inf_thoer_bound}.
\end{proof}


It is instructive to compare our approach and information-theoretic bounds with other graph learning techniques proposed in the literature. These techniques generally deal with abstract graphs not embedded into Euclidean spaces and lacking additional spatial structure. The benchmark result for Gaussian Markov networks \cite{anandkumar2012high} claims that the sample complexity scales as
\begin{equation}
n \geqslant \frac{\log\, p}{\theta^2\log\(\frac{1}{1-d\theta}\)},\quad p \to \infty.
\end{equation}

When the number of available samples is small, e.g. bounded with the growing dimension $p$ as in our case, even the moderate logarithmic dependence on the dimension is not affordable. 
Such a restriction comes from the fact that in many modern network applications the agents carry only limited memory and computational power that do not allow them to collect and process coherent data in amounts sufficient for precise structure recovery \cite{soloveychik2018region}. 
Proposition \ref{th:nec_cond_complete_graph} suggests that the spatial information available in the case of geometric graphs may allow us to learn the underlying graph structure with a constant number of samples. In the following section we develop a consistent model selection technique with this property.

\section{Approach}
\label{sec:approach}
Let us outline the basic idea of the algorithm recovering the graph structure in our setup. Later, we provide the details of its efficient implementation. 

Given a set $G \subset \mathbb{R}^2$ on the Euclidean plane, we denote by $\conv{G}$ its convex hull which is the smallest (w.r.t.\ inclusion) convex set containing $G$. We start with the following auxiliary definition.
\begin{definition}
We say that a subset $F \subset G \subset \mathbb{R}^2$ is $G$-contiguous\footnote{Since below $G$ is always clear from the context, we omit it and just say \textit{contiguous}.} if
\begin{equation}
F = G \cap \conv{F}.
\end{equation}
\end{definition}
In words, for $F$ to be contiguous means that its convex hull does not contain points of $G$ not belonging to it. The major steps of the Model Selection in Stationary Geometric Graphs (MeSSaGGe) algorithm are described in Table \ref{algo_p}.

\newenvironment{algosketch}[1][htb]
  {\renewcommand{\algorithmcfname}{Algorithm Sketch}
   \begin{algorithm}[#1]%
  }{\end{algorithm}}

\begin{algosketch}[!h]
 \caption{MeSSaGGe Algorithm Outline}
 \begin{algorithmic}[1]
 \label{algo_p}
 \renewcommand{\algorithmicrequire}{\textbf{Input:}}
 \renewcommand{\algorithmicensure}{\textbf{Output:}}
 \REQUIRE $G \subset S,\; \x_1,\dots,\x_n \in \mathbb{R}^p$;
 \ENSURE  $\widehat{\E}$; \\
 \STATE choose $\varepsilon$\tcc*{see Section \ref{sec:alg}}
 \STATE choose $r$\tcc*{see Section \ref{sec:alg}}
 \WHILE{not all $G$ is marked as \textit{detected}}
 \STATE choose a contiguous subset of $F \subset G$ on $r$ vertices;
 \STATE \label{alg:step_find} find all contiguous $\varepsilon$-copies $L_1,\dots,L_t \subset G$ of $F$;
 \STATE choose a maximal subset $\{F_1,\dots,F_q\} \subset \{L_1,\dots,L_t\}$ such that the pairwise distances between $F_i$ and $F_j$ for all $i$ and $j$ are large enough so that the samples over them can be treated as independent when $p$ grows large\tcc*{see Section \ref{sec:alg}}
 \STATE use the samples over $F_1,\dots,F_q$ to reconstruct the structure of a subgraph $\mathcal{H} \subset \mathcal{F}$ whose boundary is far enough from the boundary of $\mathcal{F}$, so that the influence of the vetrices in $G\backslash F$ can be disregarded and fill the corresponding entries of $\widehat{\E}$\tcc*{see Section \ref{sec:mod_sel}}
 \STATE mark $H$ and all its copies inside $F_1,\dots,F_q$ as \textit{detected};
 \ENDWHILE
 \end{algorithmic} 
 \end{algosketch}

Note that successful detection of the edges using the proposed technique can only be guaranteed under a proper choice of parameters $\varepsilon$ and $r$ that will ensure that the expected number of $\varepsilon$-copies of $F$ inside $G$ is large enough. Fine tuning of these parameters is achieved in Proposition \ref{prop:main_cov_res} from Section \ref{sec:alg} below. We conclude this section by counting the expected number of $\varepsilon$-copies of a subset $F \subset G$.

\begin{lemma}
\label{lem:true_num_copies}
Under assumption \ref{assm:3}, the expected number $\nu_F(\varepsilon)$ of $\varepsilon$-copies of a contiguous set $F$ on $r=|F|=o(p)$ points in $G$ is
\begin{equation}
\label{eq:exp_num_subgr}
\nu_F(\varepsilon) = \frac{2\pi\bar{l}}{\varepsilon}\(\frac{\varepsilon^2}{s^2}\)^{r-1}A_r^p \approx \frac{2\pi\bar{l}}{\varepsilon}\(\eta\varepsilon^2\)^{r-1}p,
\end{equation}
where $\bar{l}$ is the average pairwise distance between points in $F$ and
\begin{equation}
A_r^p = \frac{p!}{(p-r)!} = p\cdots(p-r+1)
\end{equation}
is the number of arrangements of $r$ elements from a set of cardinality $p$.
\end{lemma}
\begin{proof}
The proof can be found in Appendix \ref{sec:aux_res}.
\end{proof}

Remarkably, equation (\ref{eq:exp_num_subgr}) suggests that under a careful choice of $r$ and $\varepsilon$, we can achieve an almost proportional to $p$ number of copies of $F$. This will guarantee precise recovery of the graph structure in Proposition \ref{prop:consist}.

\section{Algorithm}
\label{sec:alg}
\subsection{Discretization of the Graph}
The algorithm described in Table \ref{algo_p} is computationally demanding due to Step \ref{alg:step_find} at which one has to find all $\varepsilon$-copies of $F$ in $G$. The brute force exhaustive search will become intractable already for moderate values of $p$ and $r$. In this section, we modify the MeSSaGGe algorithm so that the search becomes computationally cheaper and at the same time the sample complexity of the method does not deteriorate significantly. 

The computational improvement is reached through the approximation of $G$ by the nodes of a regular square lattice\footnote{To distinguish the vertices of graph $\mathcal{G}$ and its edges from those of the lattice, we call the latter nodes and sides.} (quantization). 
As stated in Section \ref{sec:pr_form}, we assume that the vertices of the graph $\mathcal{G}$ are uniformly distributed over the square $S$ with the side length
\begin{equation}
s = \sqrt{\frac{p}{\eta}}.
\end{equation}
Fix $0 < \varepsilon \ll s$ and cover $S$ with a regular lattice of side length $\varepsilon$, where without loss of generality we assume that $s$ is a multiple of $\varepsilon$. Next, we approximate all the elements of $G$ by the closest nodes of the constructed lattice. Note that for $\varepsilon$ small enough, such approximations are in one-to-one correspondence with the original vertices. Denote the approximation of $i \in G$ by $\tilde{i} = \tilde{i}(\varepsilon)$ and note that
\begin{equation}
\label{eq:i_dist}
\norm{i-\tilde{i}} \leqslant \frac{\varepsilon}{\sqrt{2}}.
\end{equation}
The approximating set is denoted by
\begin{equation}
\widetilde{G} = \widetilde{G}(\varepsilon) = \{\tilde{i} \mid i \in G\}.
\end{equation}

\begin{algorithm}[h]
 \caption{La-MeSSaGGe Algorithm}
 \begin{algorithmic}[1]
 \label{algo}
 \renewcommand{\algorithmicrequire}{\textbf{Input:}}
 \renewcommand{\algorithmicensure}{\textbf{Output:}}
 \renewcommand{\algorithmicfor}{\textbf{parforeach}}
 \REQUIRE $G \subset S,\; \x_1,\dots,\x_n \in \mathbb{R}^p$;
 \ENSURE  $\widehat{\E}$; \\
 \STATE $r = r(p)$\tcc*{equation (\ref{eq:set_r})}
 \STATE $\varepsilon = \varepsilon(p)$\tcc*{equation (\ref{eq:set_e})}
 \STATE $w = w(p)$\tcc*{equation (\ref{eq:w_def})}
 \STATE construct the $\varepsilon$-lattice approximation $\widetilde{G}$;
 \WHILE{not all $\widetilde{G}$ is marked as \textit{detected}} 
 \STATE \label{alg:choice_F}choose a square $K$ of size $k \times k$ nodes on the lattice containing a contiguous subgraph $\widetilde{F} \subset \widetilde{G}$ on $r$ vertices that is not all marked as \textit{detected};
 \STATE \label{alg:step_find_l} find all copies $\widetilde{L}_1,\dots,\widetilde{L}_t$ of $\widetilde{F}$ in $\widetilde{G}$;
 \STATE choose the maximal subset $\{\widetilde{F}_1,\dots,\widetilde{F}_q\} \subset \{\widetilde{L}_1,\dots,\widetilde{L}_t\}$ \\ such that $\tau(\widetilde{F}_i,\widetilde{F}_j) \geqslant w,\forall \; i\neq j$;
 \STATE collect the samples from $\widetilde{F}, \widetilde{F}_1,\dots,\widetilde{F}_q$ and calculate their SCM $\widetilde{\bm\Theta}_{\widetilde{F}}$;
 \STATE set $\widetilde{H}$ to be the graph inscribed into the middle $\lfloor \frac{k}{2} \rfloor \times \lfloor \frac{k}{2} \rfloor$ nodes square inside $K$;
 \STATE use $\widetilde{\bm\Theta}_{\widetilde{F}}$ and Lemma \ref{lem:cov_bound_inf} to identify the edges of $\widetilde{H}$ and of its copies inside $\widetilde{F}_1,\dots,\widetilde{F}_q$ and fill the corresponding entries of $\widehat{\widetilde{\E}}$\tcc*{equation (\ref{eq:shur_formula_approx})}
 \STATE mark the vertices of $\widetilde{H}$ and all its $\varepsilon$-copies inside $\widetilde{F}_1,\dots,\widetilde{F}_q$ as \textit{detected};
 \ENDWHILE
 \STATE $\widehat{\E} = \widehat{\widetilde{\E}}$;
 \end{algorithmic} 
 \end{algorithm}

%

Having constructed the approximation $\widetilde{G}$, we utilize its spatial arrangement with vertices on the lattice to reliably learn the edge structure of the corresponding graph $\widetilde{\mathcal{G}}$. Due to the one-to-one correspondence of the vertices of $\widetilde{G}$ and $G$, this will directly provide us with the desired structure of $\mathcal{G}$. The Lattice version of the MeSSaGGe algorithm is referred to as La-MeSSaGGe and its pseudo-code is provided in Table \ref{algo}. La-MeSSaGGe differs from the algorithm described in Table \ref{algo_p} by replacing $G$ with $\widetilde{G}$, and therefore, the search at Step \ref{alg:step_find_l} of the former is performed on the lattice.


For every small enough (see Proposition \ref{prop:main_cov_res}) contiguous subset $\widetilde{F} \subset \widetilde{G}$, we locate all its copies rotated by $0,\; \frac{\pi}{2},\; \pi$ and $\frac{3\pi}{2}$ radians. From the set of discovered copies we choose those that are far apart from each other, as specified below. This will ensure that the samples measured over them are close to be statistically independent (and will be asymptotically independent since we require the pairwise distances between the selected copies of $\widetilde{F}$ to bounded from below by $w(p)$ growing to infinity with $p$). We take all the samples over all the chosen copies of $\widetilde{F}$ and estimate the edge structure of the central part $\widetilde{H}$ of the latter. If the size of $\widetilde{F}$ is growing slowly enough with $p$, the amount of samples found in such manner will be enough to for the purposes of recovery.

It is important to emphasize that the set $\widetilde{F}$ chosen at Step \ref{alg:choice_F} can overlap with the previously detected regions of $\widetilde{G}$. This will ensure that the entire graph will eventually be recovered. Such procedure leads to a consistent recovery of the overlapping graph regions by Lemma \ref{lem:cov_bound_inf} below.

\begin{lemma}
\label{lem:cacl_copies}
Under assumption \ref{assm:3}, the expected number $N_{\widetilde{F}}(\varepsilon)$ of $\varepsilon$-copies of a contiguous set $\widetilde{F}$ on $r=|F| = o(p)$ points in $\widetilde{G}$ satisfies
\begin{equation}
N_{\widetilde{F}}(\varepsilon) \approx 4 p\(\eta \varepsilon^2\)^{r-1}.
\end{equation}
\end{lemma}
\begin{proof}
The proof can be found in Appendix \ref{sec:aux_res}.
\end{proof}

\subsection{Local Covariance Estimation}
The marginal distribution over the subgraph $\widetilde{\mathcal{F}}$ induced on $\widetilde{F}$ by $\widetilde{\mathcal{G}}$ is Gaussian. Let us estimate the covariance matrix corresponding to $\widetilde{\mathcal{F}}$ through the Gaussian maximum likelihood estimator, the Sample Covariance (SC) matrix,
\begin{equation}
\S_{\widetilde{\mathcal{F}}} = \frac{1}{n}\sum_{j=1} \x_{j,\widetilde{F}}\x_{j,\widetilde{F}}^\top,
\end{equation}
where we remind the reader that $\x_{j,\widetilde{F}}$ are the restriction of the $n$ measurements $\x_j$ to the subset $\widetilde{F}$. Let us now take all the copies of $\widetilde{F}$ which are least
\begin{equation}
\label{eq:w_def}
w = \varepsilon\,\log^2p
\end{equation}
apart from each other. Denote their number (including the original $\widetilde{F}$ itself) by $N'_{\widetilde{F}}(\varepsilon)$ and note that
\begin{equation}
N'_{\widetilde{F}}(\varepsilon) \geqslant \frac{N_{\widetilde{F}}(\varepsilon)}{\log^4p},
\end{equation}
therefore, due to Lemma \ref{lem:cacl_copies}
\begin{equation}
\label{eq:n_tag_bound}
N'_{\widetilde{F}}(\varepsilon) \geqslant \frac{4p}{\log^4 p}\(\eta\varepsilon^2\)^{r-1}.
\end{equation}
Next let us collect the SCs over the obtained $N'_{\widetilde{F}}(\varepsilon)$ different $\varepsilon\,\log^2p$-separated copies of $\widetilde{F}$ into a single estimate
\begin{equation}
\S_{\widetilde{F}}(\varepsilon) = \frac{1}{N'_{\widetilde{F}}(\varepsilon)}\sum_{\substack{\widetilde{\mathcal{H}} \subset \widetilde{\mathcal{G}} \\ \widetilde{\mathcal{H}} = \widetilde{\mathcal{F}} }} \S_{\widetilde{H}}.
\end{equation}
Note that when $\varepsilon = 0$, we obtain the regular SC matrix $\S_{\widetilde{\mathcal{F}}}$,
\begin{equation}
\S_{\widetilde{F}}(0) = \S_{\widetilde{F}}.
\end{equation}
In the next section, we will investigate the properties of $\S_{\widetilde{F}}(\varepsilon)$.

\subsection{Covariance Approximation and Estimation}
Recall that we need to choose a specific measure $d(\cdot,\cdot)$ to make Definition \ref{def:stat_prob} concrete. 
We will use Hellinger's distance defined as
\begin{equation}
\label{eq:helling_dist_def}
d(f_1(\x),f_2(\x)) = \frac{1}{\sqrt{2}} \norm{\sqrt{f_1(\x)}-\sqrt{f_2(\x)}}_{L^2} = \sqrt{1-\int\sqrt{f_1(\x)f_2(\x)}d\x},
\end{equation}
where $L^2$-norm is defined w.r.t. the Lebesgue measure over $\mathbb{R}^p$. Equation (\ref{eq:helling_dist_def}) implies that Hellinger's distance is symmetric, positive except for when the distributions are equal in which case it is zero, and satisfies the triangle inequality, which means it is a metric. 

For a square matrix $\bm\Theta$, denote by $\det{\bm\Theta}$ its determinant and by $\norm{\bm\Theta}$ and $\norm{\bm\Theta}_F$ its spectral and Frobenius norm, respectively.
\begin{lemma}[Exercise 1.6.14 from \cite{pardo2005statistical}]
\label{lem:hel_dist}
Hellinger's distance between two centered $r$-dimensional Gaussian distributions $p_1 \sim \mathcal{N}(\bm{0},\bm\Theta_1)$ and $p_2 \sim \mathcal{N}(\bm{0},\bm\Theta_2)$ reads as
\begin{equation}
\label{eq:hellinger_gauss}
d(p_1,p_2) = \sqrt{1-\frac{\(\det{\bm\Theta_1\bm\Theta_2}\)^{1/4}}{\(\det{\frac{\bm\Theta_1+\bm\Theta_2}{2}}\)^{1/2}}}.
\end{equation}
\end{lemma}

Denote the true covariance matrix of $\mathcal{F} \subset \mathcal{G}$ by $\bm\Theta_F$ and note that it coincides with $\bm\Theta_{\widetilde{F}}$. 

\begin{prop}
\label{prop:main_cov_res}
Set
\begin{equation}
\label{eq:set_r}
r = \log\log\,p,
\end{equation}
where we omit the rounding brackets for simplicity of notation and
\begin{equation}
\label{eq:set_e}
\varepsilon = \frac{1}{\log\,p},
\end{equation}
then for $p$ large enough with probability at least $1 - 2\exp\(-p^{0.99} n\)$,
\begin{equation}
\norm{\S_{F}(\varepsilon) - \bm\Theta_F}_F \leqslant \(\frac{1}{1-d\theta} + 2\sqrt{2}(1+d\theta)\gamma\)\frac{2}{\log\,p}.
\end{equation}
\end{prop}
\begin{proof}
The proof can be found in Appendix \ref{app:gauss_model}.
\end{proof}
This result demonstrates that in our setup it is indeed possible to learn the marginal distributions of the subgraphs of increasing sizes even with finite number of samples $n$. Moreover, the covariance estimation error decreases almost exponentially in $p$. This phenomenon will allow us to prove consistency of the MeSSaGGe algorithm in the next section.

\begin{figure*}[t!]
\centering
\includegraphics[width=15cm]{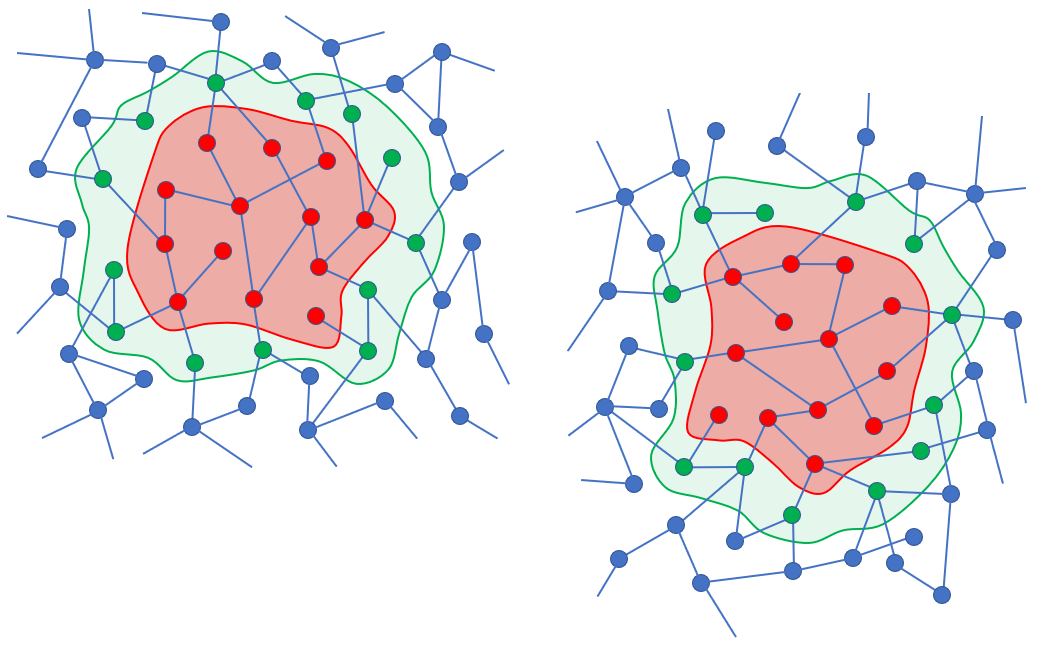}
\caption{\small A part of graph \cm{$\mathcal{G}$} (blue) with a subgraph {\color{green} $\mathcal{F}$} (green) and its rotated copy. A smaller subgraph ${\color{red} \mathcal{H}}$ (red) of ${\color{green} \mathcal{F}}$ is selected and its graph structure is estimated.}
\label{fig:subgraphs}
\end{figure*}

\section{Model Selection}
\label{sec:mod_sel}
Assume that we are given a contiguous subgraph $\mathcal{F} \subset \mathcal{G}$ on $r$ vertices with the marginal distribution $f(\mathcal{F})$ (see Figure \ref{fig:subgraphs}). Distribution $f(\mathcal{F})$ is Gaussian, however, the inverse covariance matrix will not provide us with the graph structure of $\mathcal{F}$ due to the indirect dependencies between its vertices induced by their connections through the vertices from $G \backslash F$. Recall that our graphical models satisfy CDP. We are going to utilize this property in our algorithm. 

Let us choose a contiguous subgraph $\mathcal{H} \subset \mathcal{F}$ whose vertices are far from the boundary vertices of $\mathcal{F}$, as in Figure \ref{fig:subgraphs}. This way the influence of the vertices from $G \backslash F$ on the vertices from $H$ decay proportionally to the distance between them. If we let this distance grow with the size of $F$, eventually we will be able to precisely reconstruct the edges of $\mathcal{H}$.

we denote by $\bm\Theta$ and by $\J$ the covariance and inverse covariance matrices associated with the distribution $f(\mathcal{G})$, correspondingly. Similarly, we define the pairs $\bm\Theta_F,\; \J_F$ and $\bm\Theta_H,\; \J_H$.

\begin{lemma}[Schur's Complement, \cite{zhang2006schur}]
\label{lem:shur_lemma}
Let $\bm\Theta = \J^{-1}$ be the population covariance matrix of a centered Gaussian distribution on the vertex set $G$. Denote by $F \subset G$ a subset of vertices and by $V = G \backslash F$ its complement. Partition $\J$ as
\begin{equation}
\J = \begin{pmatrix} \J_{F} & \J_{F,V} \\ \J_{F,V}^\top & \J_{V} \end{pmatrix},
\end{equation}
then
\begin{align}
\label{eq:shur_formula1}
\bm\Theta & = \begin{pmatrix} \bm\Theta_{F} & \bm\Theta_{F,V} \\ \bm\Theta_{F,V}^\top & \bm\Theta_{V} \end{pmatrix} = \J^{-1} \\
&=
\begin{pmatrix} \(\J_{F} - \J_{F,V} \J_{V}^{-1} \J_{F,V}^\top\)^{-1} & -\(\J_{F} - \J_{F,V} \J_{V}^{-1} \J_{F,V}^\top\)^{-1}\J_{F,V}\J_{V}^{-1} \\ -\J_{V}^{-1}\J_{F,V}^\top \(\J_{F} - \J_{F,V} \J_{V}^{-1} \J_{F,V}^\top\)^{-1} & \J_{V}^{-1} + \J_{V}^{-1}\J_{F,V}^\top \(\J_{F} - \J_{F,V} \J_{V}^{-1} \J_{F,V}^\top\)^{-1}\J_{F,V}\J_{V}^{-1} \end{pmatrix}. \nonumber
\end{align}
\end{lemma}

\begin{lemma}
\label{lem:cov_bound_inf}
Let $\mathcal{H} \subset \mathcal{F} \subset \mathcal{G}$ be two contiguous nested subgraphs such that the graph distance between $\mathcal{H}$ and $\mathcal{G}\backslash \mathcal{F}$ grows to infinity with $p$ as
\begin{equation}
g\(\mathcal{H},\mathcal{G}\backslash \mathcal{F}\) = \zeta(p) > 0,\quad \zeta(p) \to \infty,\; p \to \infty,
\end{equation}
then
\begin{equation}
\label{eq:shur_formula_approx}
\norm{\J_{H}^{-1} - \(\bm\Theta_{H} - \bm\Theta_{H, F\backslash H} \bm\Theta_{F\backslash H}^{-1} \bm\Theta_{H, F\backslash H}^\top\)} \leqslant c(\theta d)^{\zeta(p)+2},
\end{equation}
where $c$ is a constant independent of $\theta,\; d$ and the dimensions.
\end{lemma}
\begin{proof}
The proof can be found in Appendix \ref{app:gauss_model}.
\end{proof}

As a corollary, we get the following statement.
\begin{prop}
\label{prop:consist}
Let the assumptions \ref{assm:1}-\ref{assm:3} hold, then the La-MeSSaGGe algorithm (Table \ref{algo}) is consistent as $p\to \infty$.
\end{prop}
\begin{proof}
Note that at Step \ref{alg:choice_F} in Table \ref{algo}, the number of nodes on each size of the chosen square is bounded by
\begin{equation}
k \leqslant \frac{1}{\varepsilon}\sqrt{\frac{r}{\eta}}\log\,r,
\end{equation}
with high probability, since the distribution of the vertices over $S$ is assumed to be uniform. As a consequence of the choice of $\varepsilon$ in (\ref{eq:set_e}),
\begin{equation}
k = o(\log^3p).
\end{equation}
This implies that the number of all possible discretized graphs contained by such a square is bounded by
\begin{equation}
\label{eq_d_sq_bound}
D(\varepsilon) \leqslant 2^{k(k-1)/2} \leqslant 2^{\log^6 p}.
\end{equation}
In order to show that our La-MeSSaGGe algorithm is consitent, we need to show that the number of mistakenly detected or undetected edges vanishes with growing $p$. Let us analyze the probability of such an error when detecting over one square with $k \times k$ nodes. In this case, a mistake in detection will lead to the difference between the true inverse covariance and the estimated one by at least one entry, in which case
\begin{equation}
\norm{\widehat{\bm\Theta^{-1}} - \bm\Theta^{-1}}_F \geqslant \theta.
\end{equation}
We want to estimate the probability of such event over each possible square and then bound the overall probability of error using the union bound. Note that for small enough $\norm{\widehat{\bm\Theta^{-1}} - \bm\Theta^{-1}}_F$,
\begin{equation}
\label{eq:theta_diff}
\norm{\widehat{\bm\Theta^{-1}} - \bm\Theta^{-1}}_F \leqslant \frac{1}{(1-d\theta)^2}\norm{\widehat{\bm\Theta} - \bm\Theta}_F,
\end{equation}
since the spectrum of $\bm\Theta^{-1}$ is bounded from below by $(1-d\theta)$. Recall, that we use $\S_{F}(\varepsilon)$ to estimate the covariance matrix. Thus, $\norm{\widehat{\bm\Theta} - \bm\Theta}_F$ in (\ref{eq:theta_diff}) can be bounded using Proposition \ref{prop:main_cov_res},
\begin{equation}
\mathbb{P}\[\norm{\S_{F}(\varepsilon) - \bm\Theta_F}_F \geqslant \(\frac{1}{1-d\theta} + 2\sqrt{2}(1+d\theta)\gamma\)\frac{2}{\log\,p}\] \leqslant 2\exp(-p^{0.99} n).
\end{equation}
Taking into consideration the bound (\ref{eq_d_sq_bound}) on the number of all possibles graphs inscribed into a $k \times k$ square on our lattice and using the union bound, we conclude that the probability of detecting or undetecting an edge erroneously in $\widetilde{\mathcal{G}}$ is bounded by
\begin{equation}
2^{\log^6 p} 2\exp(-p^{0.99} n) = o(1),\quad p \to \infty,
\end{equation}
which implies consistency and concludes the proof.
\end{proof}


\section{Conclusions}
\label{sec:concl}
In this article, we consider the problem of model selection in Gaussian Markov fields in the sample deficient scenario. We introduce the notion of spatially stationary distributions and apply it to graphical models embedded into Euclidean spaces. The spatial structure and stationarity allow us to significantly reduce the number of samples necessary for graph recovery. We develop an efficient model selection algorithm for graphs with vertices uniformly distributed over a plane and edges of bounded length and show that even finite number of samples is sufficient for consistent learning.

\section{Acknowledgment}
The authors would like to thank Dmitry Trushin and Roman Vershynin for their helpful comments and remarks.

\appendices
\section{Information-Theoretic Lower Bound for Model Selection}
\label{app:graph_detect_inf_thoer_bound}
Following the approach developed in \cite{anandkumar2012high}, in this Section we prove Proposition \ref{th:nec_cond_complete_graph} using Fano's inequality. We shall require a number of auxiliary results stated below.

Denote the Kullback-Leibler (KL) divergence between two probability measures $\mathbb{P}_{\mu_i}$ and $\mathbb{P}_{\mu_j}$ associated with different parameters $\mu_i \neq \mu_j$ by
\begin{equation}
D(\mu_i\;\|\;\mu_j) = \mathbb{E}_{\mu_i} \[\log \frac{\mathbb{P}_{\mu_i}}{\mathbb{P}_{\mu_j}}\].
\end{equation}
Next we introduce the symmetrized KL-divergence,
\begin{equation}
\label{eq:sym_kl_div}
S(\mu_i\;\|\;\mu_j) = D\(\mu_i\;\|\;\mu_j\) + D\(\mu_j\;\|\;\mu_i\).
\end{equation}

\begin{definition}
\label{def:delta-unrel}
Let $M = \{\mu_1,\dots,\mu_{|M|}\}$ be a family of models and $\x_i \sim \mathbb{P}_{\mu_j}$ be i.i.d.\ for some $\mu_j \in M$. Denote by
\begin{equation}
\psi \colon \{\mathcal{X}^n\} \to M
\end{equation}
a classification function (decoder) where $\mathcal{X}^n = \{\x_1,\dots,\x_n\}$. We say that $\psi$ is $\delta$-unreliable if
\begin{equation}
\max_j \mathbb{P}_{\mu_j} \[\psi(\mathcal{X}^n) \neq \mu_j\] \geqslant \delta - \frac{1}{\log\, |M|}.
\end{equation}
\end{definition}

\begin{lemma}[Fano's Inequality, \cite{yu1997assouad}]
\label{lem:fano_ineq_lemma}
In the setup of Definition \ref{def:delta-unrel}, if the number of i.i.d.\ samples is bounded as
\begin{equation}
\label{eq:fano_2}
n < (1-\delta)\frac{\log |M|}{\frac{2}{{|M|}^2}\sum\limits_{i=1}^{|M|}\sum\limits_{j=i+1}^{|M|} S(\mu_i\;\|\;\mu_j)}.
\end{equation}
where $\mathcal{U}(M)$ is the uniform distribution over $M$, then any decoder over $M$ is $\delta$-unreliable.
\end{lemma}

\begin{lemma}
\label{lem:gauss_area_diff}
Given a graph $G_p$, let $\bm\theta_1$ and $\bm\theta_2$ be two $\frac{p(p-1)}{2}$ dimensional coupling parameter vectors corresponding to two Gaussian graphical models over the same graph. Denote their respective precision matrices by $\J_1$ and $\J_2$, then
\begin{equation}
\label{eq:trace_lemma}
S(\bm\theta_1\;\|\;\bm\theta_2) = \frac{1}{2}\Tr{\(\J_1-\J_2\)\(\J_2^{-1}-\J_1^{-1}\)},
\end{equation}
where $S(\bm\theta_1\;\|\;\bm\theta_2)$ is defined in (\ref{eq:sym_kl_div}).
\end{lemma}

\begin{proof}
Denote the corresponding covariance matrices by $\bm\Sigma_1=\J_1^{-1}$ and $\bm\Sigma_2=\J_2^{-1}$. The symmetrized KL-divergence between two normal distributions $\mathbb{P}_1 = \mathcal{N}(\bm{0},\bm\Sigma_2)$ and $\mathbb{P}_2 = \mathcal{N}(\bm{0},\bm\Sigma_1)$ reads as
\begin{align}
S(\bm\theta_1\;\|\;\bm\theta_2) & = \int\log\frac{\mathbb{P}_1}{\mathbb{P}_2} d\mathbb{P}_1 + \int\log\frac{\mathbb{P}_2}{\mathbb{P}_1} d\mathbb{P}_2 \\
&= \int \log\[\frac{\sqrt{\det{\bm\Sigma_2}}\exp\(-\frac{1}{2}\x^\top\bm\Sigma_1^{-1}\x\)}{\sqrt{\det{\bm\Sigma_1}}\exp\(-\frac{1}{2}\x^\top\bm\Sigma_2^{-1}\x\)}\] d\mathbb{P}_1 + \int \log\[\frac{\sqrt{\det{\bm\Sigma_1}}\exp\(-\frac{1}{2}\x^\top\bm\Sigma_2^{-1}\x\)}{\sqrt{\det{\bm\Sigma_2}}\exp\(-\frac{1}{2}\x^\top\bm\Sigma_1^{-1}\x\)}\] d\mathbb{P}_2 \nonumber \\
& = \frac{1}{2}\[\mathbb{E}_1\[\Tr{\x\x^\top\bm\Sigma_2^{-1}}-\Tr{\x\x^\top\bm\Sigma_1^{-1}}\] + \mathbb{E}_2\[\Tr{\x\x^\top\bm\Sigma_1^{-1}}-\Tr{\x\x^\top\bm\Sigma_2^{-1}}\]\] \nonumber \\
& = \frac{1}{2}\[\Tr{\bm\Sigma_1\(\bm\Sigma_2^{-1}-\bm\Sigma_1^{-1}\)} + \Tr{\bm\Sigma_2\(\bm\Sigma_1^{-1}-\bm\Sigma_2^{-1}\)}\] \nonumber \\
& = \frac{1}{2}\Tr{\(\bm\Sigma_2^{-1}-\bm\Sigma_1^{-1}\)\(\bm\Sigma_1-\bm\Sigma_2\)} = \frac{1}{2}\Tr{\(\J_1-\J_2\)\(\J_2^{-1}-\J_1^{-1}\)}.\nonumber
\end{align}
\end{proof}

\begin{corollary}
\label{cor:gauss_area_diff_1}
Let $\mathcal{G}_1$ and $\mathcal{G}_2$ be two graphs satisfying the assumptions from Section \ref{sec:stat_gr} and their corresponding parameters be their adjacency matrices $\E_1$ and $\E_2$, then
\begin{equation}
S(\E_1\;\|\;\E_2) \leqslant pd\(\frac{\theta}{1-d\theta}\)^2.
\end{equation}
\end{corollary}
\begin{proof}
Let $\B$ and $\C$ be two invertible matrices of the same size. Consider the identity
\begin{equation}
\B^{-1}-\C^{-1} = \B^{-1}\(\C-\B\)\C^{-1},
\end{equation}
which can be easily checked by multiplying by $\B$ on the left and by $\C$ on the right. Using this identity, we rewrite the right-hand side of (\ref{eq:trace_lemma}) as
\begin{equation}
\Tr{\(\J_1-\J_2\)\(\J_2^{-1}-\J_1^{-1}\)} = \Tr{\(\J_1-\J_2\)\J_2^{-1}\(\J_1-\J_2\)\J_1^{-1}}.
\end{equation}
We can bound the norm of $\J_1^{-1}$ as
\begin{equation}
\label{eq:norm_bound}
\norm{\J_1^{-1}} \leqslant \frac{1}{1-\theta\norm{\A_1}} \leqslant \frac{1}{1-d\theta}.
\end{equation}
Similar bound holds for $\norm{\J_2^{-1}}$, as well. Taking into account the inequality
\begin{equation}
\label{eq:trace_ineq_1}
\Tr{\(\J_1-\J_2\)\(\J_2^{-1}-\J_1^{-1}\)} \leqslant \frac{1}{\(1-d\theta\)^2} \Tr{\(\J_1-\J_2\)^2},
\end{equation}
finally we conclude that
\begin{equation}
\Tr{\(\J_1-\J_2\)\(\J_2^{-1}-\J_1^{-1}\)} \leqslant \frac{1}{\(1-d\theta\)^2} 2pd\theta^2,
\end{equation}
since the number of non-zero elements in $\J_1-\J_2$ is at most $2pd$. Now the statement of the corollary follows.
\end{proof}

\begin{lemma}[Asymptotic Enumeration of Labeled $d$-regular Graphs on $k$ Vertices, \cite{mckay1991asymptotic}]
\label{lem:gr_count}
Let $d=o(\sqrt{p})$, then the number of labeled regular graphs of degree $d$ on $k$ nodes is
\begin{equation}
\label{eq:reg_gr_num}
\mathcal{Q}_d(k) = \frac{(kd)!}{\(\frac{kd}{2}\)!2^{kd/2}(d!)^k}\exp\(-\frac{d^2-1}{4}-\frac{d^3}{12k}+O\(\frac{d^2}{k}\)\),\quad k \to \infty.
\end{equation}
\end{lemma}

\begin{lemma}
\label{lem:graph_lemma}
For large enough $p$, the cardinality of the set of admissible models $|\mathfrak{H}|$ satisfies
\begin{equation}
\log_2|\mathfrak{H}| \geqslant \frac{dp}{2}\log_2 \(\frac{\eta \beta^2}{d}\).
\end{equation}
\end{lemma}
\begin{proof}
To get a lower bound on the number of admissible graphs in $\mathfrak{H}$, let us cover the square $S$ with a lattice of side length $\beta$ and count the total number of $d$-regular graphs on vertices $G$ whose edges do not cross the lattice, or in other words belong to the cells. From (\ref{eq:beta_assm}), the average number of vertices inside a cell can be calculated as
\begin{equation}
\label{eq:k_def}
k = \eta \beta^2 \gg 1.
\end{equation}
Since $k \gg 1$, we get a good lower bound on the total number of possible graphs by assuming that inside each cell the subgraph can be chosen arbitrarily from the family of $d$-regular graphs on $k$ vertices. The number of such subgraphs is given by formula (\ref{eq:reg_gr_num}), and overall we get
\begin{equation}
\label{eq:comb_b_t1}
|\mathfrak{H}| \geqslant \(\mathcal{Q}_d(k)\)^{p/k}.
\end{equation}
Using Stirling's approximation
\begin{equation}
m! \sim \sqrt{2\pi m}\(\frac{m}{e}\)^m,\quad m \to \infty,
\end{equation}
we obtain for $k \gg d$
\begin{align}
\(\mathcal{Q}_d(k)\)^{p/k} &\geqslant \frac{(kd)!}{\(\frac{kd}{2}\)!2^{kd/2}(d!)^{k}}\exp\(-\frac{d^2-1}{4}-\frac{d^3}{12k}+O\(\frac{d^2}{k}\)\) \nonumber \\
& \geqslant \frac{\sqrt{2\pi kd}(kd)^{kd}e^{kd/2}e^{dk}}{e^{kd}\sqrt{\pi kd}\(\frac{kd}{2}\)^{\frac{kd}{2}}2^{kd/2}(\sqrt{2\pi d})^{k}d^{dk}}\exp\(-\frac{d^2-1}{4}-\frac{d^3}{12k}+O\(\frac{d^2}{k}\)\) \nonumber\\
& \geqslant \(\frac{ke}{d}\)^{kd/2}\frac{1}{(2\pi d)^{k/2}}\exp\(-\frac{d^2-1}{4}-\frac{d^3}{12k}+O\(\frac{d^2}{k}\)\) \nonumber \\
&\geqslant \(\frac{ke}{d}\)^{k(d-1)/2}.
\end{align}
Plug the last inequality into (\ref{eq:comb_b_t1}) to get
\begin{equation}
\label{eq:comb_b_t}
|\mathfrak{H}| \geqslant \(\mathcal{Q}_d(k)\)^{p/k} \geqslant \(\frac{ke}{d}\)^{p(d-1)/2},
\end{equation}
which completes the proof.
\end{proof}

\begin{proof}[Proof of Proposition \ref{th:nec_cond_complete_graph}]
From Corollary \ref{cor:gauss_area_diff_1} we conclude that for the family of admissible graphs,
\begin{equation}
\frac{2}{M^2}\sum_{i=1}^M\sum_{j=i+1}^M S(\E_i\;\|\;\E_j) \leqslant pd\(\frac{\theta}{1-d\theta}\)^2.
\end{equation}
Plugging this bound and the result of Lemma \ref{lem:graph_lemma} into Lemma \ref{lem:fano_ineq_lemma} we obtain that if
\begin{equation}
\label{eq:h_low_b_unc}
n < (1-\delta)\frac{\log_2 (|\mathfrak{H}|)}{pd\(\frac{\theta}{1-d\theta}\)^2} \leqslant (1-\delta)\frac{\frac{dp}{2}\log\(\frac{\eta \beta^2}{d}\)}{pd\(\frac{\theta}{1-d\theta}\)^2} = (1-\delta)\frac{\log\(\frac{\eta \beta^2}{d}\)}{2\(\frac{\theta}{1-d\theta}\)^2},
\end{equation}
then any decoder will be $\delta$-unreliable. Since we are interested in vanishing errors, the claim follows.
\end{proof}

\section{Random Geometric Sets}
\label{sec:aux_res}

To make the proofs of this section less technically involved and without loss of precision, assume that the sides of the square $S$ are glued and its surface becomes a torus. This way we avoid separate treatment for the boundary of $S$.

\begin{proof}[Proof of Lemma \ref{lem:true_num_copies}]
Enumerate the points of $F$ from $f_1$ to $f_r$. Throw a point $h_1$ uniformly at random onto the torus $S$ and move $F$ to align $f_1$ with $h_1$,
\begin{equation}
\label{eq:point_trow}
h_1=f_1.
\end{equation}
Now throw another point $h_2$ onto $S$ at random and note that under the condition (\ref{eq:point_trow}), we still have one degree of freedom in placing the template $F$, namely we can rotate it and try to align $f_2$ with $h_2$. The probability that we will succeed is $\frac{2\pi\norm{f_1-f_2}\varepsilon}{s^2}$. If we successfully aligned the first two pints, we have no more freedom in translating or rotating $F$ and the rest of points $h_3,\dots,h_r$ have to fall into the discs of radii $\varepsilon$ around the points $f_3,\dots,f_r$ respectively. The overall probability of successful alignment can be written as
\begin{equation}
\label{eq:prob_f1q}
P_{f_1,\dots,f_r}(\varepsilon) = \frac{2\pi\norm{f_1-f_2}\varepsilon}{s^2}\(\frac{\pi\varepsilon^2}{s^2}\)^{r-2}.
\end{equation}
Denote the detected $\varepsilon$-copy of $F$ by $H$ and note that we need to ensure that the convex hull of $H$ does not contain any other points except for the detected ones. Since the area of $F$ vanishes compared to the area of $S$, this will hold with probability approaching one, and therefore will not affect (\ref{eq:prob_f1q}).

Denote the average distance between the points in $F$ by $\bar{l}$ and note that the overall number of combinations of enumerated $r$-tuples from $p$ points is given by $A_r^p$, we conclude that
\begin{equation}
\nu_F(\varepsilon) = 2\pi\bar{l}\,\frac{\varepsilon^{2r-3}}{s^{2r-2}}A_r^p \approx 2\pi\bar{l}\,\frac{\varepsilon^{2r-3}}{s^{2r-2}}p^r = \frac{2\pi\bar{l}}{\varepsilon}\(\eta\varepsilon^2\)^{r-1}p.
\end{equation}
\end{proof}

\begin{proof}[Proof of Lemma \ref{lem:cacl_copies}]
Denote
\begin{equation}
L = \frac{s}{\varepsilon} + 1.
\end{equation}
Let us start from the case when the circumscribed rectangle of $F$ is a square $K$ on $k \times k$ nodes of side length $(k-1)\varepsilon$.

Associate with every node of the lattice over $S$ a random indicator $X_{j,l},\; j,l=1,\dots,L$ such that
\begin{equation}
X_{j,l} = \begin{cases} 
1, & \text{if } (j,l) \in \widetilde{G} \text{ is occupies by a node}, \\
0, & \text{otherwise}.
\end{cases}
\end{equation}
Similarly, with every node of the $\varepsilon$-lattice over $K$ we associate an indicator $Y_{j,l},\; j,l=1,\dots,k$. Denote the set (pattern) of indicators over $K$ by
\begin{equation}
\widetilde{F} = \{Y_{j,l}\}_{j,l=1}^k.
\end{equation}
Let us count the expected number of occurrences of shifted $\widetilde{F}$ in $\widetilde{G}$. Indeed, by the linearity of expectation,
\begin{multline}
N_{\widetilde{F}}(\varepsilon) = \mathbb{E}\(\text{number of occurrences of } \widetilde{F} \text{ in } \widetilde{G}\) \\ 
= \mathbb{E}\(\sum_{j,l=1}^{L}\[X_{j,l}\dots X_{j+k-1,l+k-1} = Y_{1,1}\dots Y_{k,k}\]\),
\end{multline}
where $\[\cdot\]$ are Iverson brackets converting a logical proposition into $1$ if the proposition is satisfied, and $0$ otherwise. We obtain,
\begin{align}
N_{\widetilde{F}}(\varepsilon) &= 4\sum_{j,l=1}^{L} \mathbb{P}\(X_{j,l}\dots X_{j+k-1,l+k-1} = Y_{1,1}\dots Y_{k,k}\) \nonumber \\
&=4L^2 \mathbb{P}\(X_{1,1}\dots X_{k,k} = Y_{1,1}\dots Y_{k,k}\) \nonumber \\
&\stackrel{(i)}{\approx} 4L^2 \mathbb{P}\(\bigcap_{j,l=1}^k X_{j,l} = Y_{j,l}\) \label{eq:approx_ineq} \\
&=4L^2 \prod_{j,l=1}^k \mathbb{P}\(X_{j,l} = Y_{j,l}\) \nonumber \\
&=4L^2 \frac{{L^2-k^2 \choose p-r}}{{L^2 \choose p}} \nonumber \\
&= 4L^2 \frac{(L^2-k^2)!p!(L^2-p)!}{(p-r)!(L^2-k^2-p+r)!(L^2)!} \nonumber \\
& \stackrel{(ii)}{\approx} 4L^2\(\frac{p}{L^2}\)^r, \nonumber
\end{align}
where the multiplier $4$ comes from counting all possible rotations of $\widetilde{F}$. In (ii) we assume that
\begin{equation}
\label{eq:lim_asympt}
r = o(p),\;\; k = o(p),\quad p \to \infty.
\end{equation}
Note also that under (\ref{eq:lim_asympt}) the approximation (i) in (\ref{eq:approx_ineq}) asymptotically becomes an equality since $X_{j,l},\dots,X_{j+k,l+k}$ are asymptotically independent. Recall (\ref{eq:density}) to get
\begin{equation}
\label{eq:n_exp_calc}
N_{\widetilde{F}}(\varepsilon) \approx 4p\(\eta\varepsilon^2\)^{r-1}.
\end{equation}

Clearly, in the case of general circumscribed rectangle of $\widetilde{F}$ the calculation is very similar and asymptotically the result coincides with (\ref{eq:n_exp_calc}).
\end{proof}

\section{Gaussian Model Selection}
\label{app:gauss_model}

\begin{lemma}
\label{cor:hel_dist}
In the setup of Lemma \ref{lem:hel_dist}, let $\bm\Theta_2 = \bm\Theta_1 + \Delta\bm\Theta$, then
\begin{equation}
d(p_1,p_2) = \frac{\norm{\bm\Theta_1^{-1}\Delta\bm\Theta}_F}{4} + O\(r^{3/2}\norm{\bm\Theta_1^{-1}\Delta\bm\Theta}_F^2\).
\end{equation}
In particular if
\begin{equation}
\label{eq:cond_eps}
\norm{\bm\Theta_1^{-1}\Delta\bm\Theta}_F^2 = o\(r^{-3/2}\),\quad r \to \infty,
\end{equation}
then for $r$ large enough,
\begin{equation}
\frac{\norm{\bm\Theta_1^{-1}\Delta\bm\Theta}_F}{8} \leqslant d(p_1,p_2) \leqslant 3\frac{\norm{\bm\Theta_1^{-1}\Delta\bm\Theta}_F}{8}.
\end{equation}
\end{lemma}
\begin{proof} 
In our case, equation (\ref{eq:hellinger_gauss}) reads as
\begin{equation}
\label{eq:dist_dev}
d(p_1,p_2) = \sqrt{1-\frac{\(\det{\bm\Theta_1}\det{\bm\Theta_1+\Delta\bm\Theta}\)^{1/4}}{\(\det{\bm\Theta_1 + \frac{\Delta\bm\Theta}{2}}\)^{1/2}}}.
\end{equation}
Taylor's expansion up to the second order term of the determinant function reads as
\begin{align}
\label{eq:det_expr}
\det{\bm\Theta_1+\Delta\bm\Theta} & = \det{\bm\Theta_1}\(1 + \Tr{\A} + \frac{1}{2}\(\(\Tr{\A}\)^2 - \Tr{\(\A\)^2}\)+O|\Tr{\A}|^3\) \nonumber \\
& = \det{\bm\Theta_1}\(1 + \Tr{\A} + \frac{1}{2}\(\(\Tr{\A}\)^2 - \Tr{\(\A\)^2}\)+O\(r^{3/2}\norm{\A}_F^3\)\),
\end{align}
where
\begin{equation}
\A = \bm\Theta_1^{-1}\Delta\bm\Theta.
\end{equation}
Using expression (\ref{eq:det_expr}), the numerator of the fraction in (\ref{eq:dist_dev}) can be written as
\begin{align}
\label{eq:numer_det}
&\(\det{\bm\Theta_1}\det{\bm\Theta_1+\Delta\bm\Theta}\)^{1/4} = \(\det{\bm\Theta_1}^2\(1 + \Tr{\A} + \frac{1}{2}\(\(\Tr{\A}\)^2 - \Tr{\A^2}\) + O\(r^{3/2}\norm{\A}_F^3\)\)\)^{1/4} \nonumber \\
& \qquad = \(\det{\bm\Theta_1}^2\(1 + \Tr{\A} + \frac{1}{2}\(\Tr{\A}\)^2 - \frac{1}{2}\norm{\A}_F^2 + O\(r^{3/2}\norm{\A}_F^3\)\)\)^{1/4}.
\end{align}
Similarly, for the denominator we obtain,
\begin{align}
\label{eq:denom_det}
& \(\det{\bm\Theta_1 + \frac{\Delta\bm\Theta}{2}}\)^{1/2} = \(\det{\bm\Theta_1}\(1 + \frac{1}{2}\Tr{\A} + \frac{1}{8}\(\(\Tr{\A}\)^2 - \Tr{\A^2}\) + O\(r^{3/2}\norm{\A}_F^3\)\)\)^{1/2} \nonumber \\
& \quad = \(\det{\bm\Theta_1}^2\(1 +  \Tr{\A} + \frac{1}{4}\(\(\Tr{\A}\)^2 - \Tr{\A^2}\) + \frac{1}{4}\(\Tr{\A}\)^2 + O\(r^{3/2}\norm{\A}_F^3\)\)\)^{1/4} \nonumber \\
& \quad = \(\det{\bm\Theta_1}^2\(1 + \Tr{\A} + \frac{1}{2}\(\Tr{\A}\)^2 - \frac{1}{4}\norm{\A}_F^2 + O\(r^{3/2}\norm{\A}_F^3\)\)\)^{1/4}.
\end{align}
Plug (\ref{eq:numer_det}) and (\ref{eq:denom_det}) into (\ref{eq:dist_dev}) to obtain
\begin{align}
d(p_1,p_2) & = \sqrt{1-\(1 - \frac{1}{4}\norm{\A}_F^2 + O\(r^{3/2}\norm{\A}_F^3\)\)^{1/4}} = \sqrt{\frac{1}{16}\norm{\A}_F^2 + O\(r^{3/2}\norm{\A}_F^3\)} \nonumber \\ 
& = \sqrt{\frac{1}{16}\norm{\A}_F^2\(1 + O\(r^{3/2}\norm{\A}_F\)\)} = \frac{\norm{\A}_F}{4} + O\(r^{3/2}\norm{\A}_F^3\).
\end{align}
\end{proof}

\begin{lemma}
\label{prop:theta_approx}
Let the assumptions of Section \ref{sec:stat_gr} hold and $\widetilde{\mathcal{H}}$ be (a possibly rotated) copy of $\widetilde{\mathcal{F}}$, then the true covariance matrices $\bm\Theta_F$ and $\bm\Theta_H$ of the subgraphs $\mathcal{F}, \mathcal{H} \subset \mathcal{G}$ satisfy
\begin{equation}
\norm{\bm\Theta_F^{-1}\(\bm\Theta_F-\bm\Theta_H\)}_F \leqslant 4\sqrt{2}\gamma\varepsilon.
\end{equation}
\end{lemma}
\begin{proof}
The claim follows from Definition \ref{def:stat_prob}, Lemma \ref{cor:hel_dist}, and equation (\ref{eq:i_dist}).
\end{proof}

\begin{lemma}
\label{lem:frob_concentr}
Let $\x_1,\dots,\x_n$ be i.i.d.\ copies of $\x \sim \mathcal{N}(0,\bm\Theta)$, where $\bm\Theta$ is an $r \times r$ positive definite matrix. Then for the SC matrix $\S = \frac{1}{n}\sum_{i=1}^n \x_i\x_i^\top$,
\begin{equation}
\mathbb{P}\[\norm{\S-\bm\Theta}_F \geqslant \norm{\bm\Theta}\(\frac{r}{\sqrt{n}} + \delta\) \] \leqslant 1 - 2e^{-n \delta^2/2r}.
\end{equation}
\end{lemma}
\begin{proof} 
Let us start from the case $\bm\Theta = \I$. Due to Corollary 5.35 from \cite{vershynin2012introduction},
\begin{equation}
\mathbb{P}\[\norm{\S-\I} \geqslant \sqrt{\frac{r}{n}} + \delta \] \leqslant 2e^{-n \delta^2/2}.
\end{equation}
Recall that
\begin{equation}
\norm{\A}_F \leqslant \sqrt{r}\norm{\A},
\end{equation}
where $\A$ is an arbitrary $r \times r$ matrix to conclude,
\begin{equation}
\mathbb{P}\[\norm{\S-\I}_F \geqslant \frac{r}{\sqrt{n}} + \delta \] \leqslant 2e^{-n \delta^2/2r},
\end{equation}
now the statement follows.
\end{proof}

\begin{corollary}
\label{cor:frob_concentr}
Let $\x_1^k,\dots,\x_n^k,\; k=1,\dots,m$ be copies of $\x^k \sim \mathcal{N}(0,\bm\Theta_k)$ where $\bm\Theta_k$ are $r \times r$ positive definite matrices and $\x_i^k$ are independent for all $i$ and all $k$. Then for the SC matrices $\S_k = \frac{1}{n}\sum_{i=1}^n \x_i^k\x_i^{k\top}$,
\begin{equation}
\mathbb{P}\[\norm{\frac{1}{m}\sum_{k=1}^m \(\S_k-\bm\Theta_k\)}_F \geqslant \max_k\norm{\bm\Theta_k}\(\frac{r}{\sqrt{mn}} + \delta\) \] \leqslant 1 - 2e^{-mn \delta^2/2r}.
\end{equation}
\end{corollary}
\begin{proof}
Following \cite{vershynin2012introduction}, the application of Talagrand's technique from \cite{davidson2001local} gives the result.
\end{proof}

\begin{proof}[Proof of Proposition \ref{prop:main_cov_res}]
Recall that in our search for similar subsets of $\widetilde{G}$, we only took $N'_{\widetilde{F}}(\varepsilon)$ of them that are at least $\varepsilon\log^2p$-far from each other. Since our graph possesses CDP, the restrictions of a sample $\x_i$ onto the chosen subsets are asymptotically independent. Thus, we obtain
\begin{align}
\label{eq:long_eq}
\norm{\S_{\widetilde{F}}(\varepsilon) - \bm\Theta_{F}}_F &= \norm{\frac{1}{N'_{\widetilde{F}}(\varepsilon)}\sum_{\substack{\widetilde{\mathcal{H}} \subset \widetilde{\mathcal{G}}, \; \widetilde{H} = \widetilde{F}, \\ \norm{\widetilde{H}-\widetilde{F}} \geqslant \varepsilon\log^2p}} \S_{\widetilde{H}} - \bm\Theta_F}_F \\
&= \norm{\frac{1}{N'_{\widetilde{F}}(\varepsilon)}\sum \(\S_{\widetilde{H}} - \bm\Theta_H\) + \frac{1}{N'_{\widetilde{F}}(\varepsilon)}\sum \(\bm\Theta_H - \bm\Theta_F\)}_F \nonumber \\
&\leqslant \norm{\frac{1}{N'_{\widetilde{F}}(\varepsilon)}\sum \(\S_{\widetilde{H}} - \bm\Theta_H\)}_F + \frac{1}{N'_{\widetilde{F}}(\varepsilon)} \sum \norm{\bm\Theta_H - \bm\Theta_F}_F. \nonumber
\end{align}
Let us now separately bound the two summands in the last line of (\ref{eq:long_eq}). Corollary \ref{cor:frob_concentr} implies that
\begin{equation}
\label{eq:n_tag_b}
\norm{\frac{1}{N'_{\widetilde{F}}(\varepsilon)}\sum \(\S_{\widetilde{H}} - \bm\Theta_H\)}_F \leqslant \max_H\norm{\bm\Theta_H}\(\frac{r}{\sqrt{N'_{\widetilde{F}}(\varepsilon)n}} + \delta\) \leqslant \frac{1}{1-d\theta}\(\frac{r}{\sqrt{N'_{\widetilde{F}}(\varepsilon)n}} + \delta\),
\end{equation}
with probability at least $1 - 2e^{-N'_{\widetilde{F}}(\varepsilon)n \delta^2/2r}$. Next we apply Lemma \ref{prop:theta_approx} to bound the remaining term in the last line of (\ref{eq:long_eq}). Below we show that the condition (\ref{eq:cond_eps}) is satisfied by the choice of $r$ and $\varepsilon$ which justifies application of Lemma \ref{prop:theta_approx} here. We obtain,
\begin{equation}
\frac{1}{N'_{\widetilde{F}}(\varepsilon)} \sum \norm{\bm\Theta_H - \bm\Theta_F}_F \leqslant \max\norm{\bm\Theta_{H}^{-1}}4\sqrt{2}\gamma\varepsilon \leqslant (1+d\theta)4\sqrt{2}\gamma\varepsilon.
\end{equation}

Set
\begin{equation}
\delta = \frac{1}{\log\, p}
\end{equation}
and recall that according to the choice made in (\ref{eq:set_r}) and (\ref{eq:set_e}),
\begin{equation}
r = \log\log\,p,
\end{equation}
\begin{equation}
\label{eq:eps_delt_ch}
\varepsilon = \frac{1}{\log\, p}.
\end{equation}
As a consequence,
\begin{equation}
N'_{\widetilde{F}}(\varepsilon) \geqslant \frac{4p}{\log^4 p}\(\eta\varepsilon^2\)^{r-1} = \frac{4p\eta^{\log\log p-1}}{\log^{2\log\log p+2}p} \geqslant p^{1-\phi},
\end{equation}
for any $\phi > 0$. This implies,
\begin{equation}
\frac{r}{\sqrt{N'_{\widetilde{F}}(\varepsilon)n}} = o\(\frac{1}{\log\, p}\),
\end{equation}
and therefore, from (\ref{eq:n_tag_b}) we get
\begin{equation}
\norm{\frac{1}{N'_{\widetilde{F}}(\varepsilon)}\sum \(\S_{\widetilde{H}} - \bm\Theta_H\)}_F \leqslant \frac{2}{(1-d\theta)\log\,p},
\end{equation}
with probability at least
\begin{equation}
\label{eq:prob_b2}
1 - 2\exp\(-\frac{N'_{\widetilde{F}}(\varepsilon)n \delta^2}{2r}\) \geqslant 1 - 2\exp\(-\frac{p^{1-\phi} n}{2\log^2 p\,\log\log\,p}\) \geqslant 1 - 2\exp\(-p^{1-2\phi} n\).
\end{equation}
Due to the choice of $\varepsilon$ in (\ref{eq:set_e}), at least with the same probability,
\begin{equation}
\label{eq:deriv_31}
\norm{\S_{F}(\varepsilon) - \bm\Theta_F}_F \leqslant \(\frac{1}{1-d\theta} + 2\sqrt{2}(1+d\theta)\gamma\)\frac{2}{\log\,p}.
\end{equation}
Note that the condition (\ref{eq:cond_eps}) in satisfied justifying the above application of Corollary \ref{prop:theta_approx}. This concludes the proof.
\end{proof}

\begin{lemma}[CDP Lemma]
\label{lem:cdp}
Let $\mathcal{H} \subset \mathcal{F} \subset \mathcal{G}$ be two nested subgraphs such that the graph distance between $\mathcal{H}$ and $\mathcal{G}\backslash \mathcal{F}$ is at least $\zeta$, then
\begin{equation}
\norm{\J_{H, F\backslash H} \J_{F\backslash H}^{-1} \J_{F\backslash H, G\backslash F}} \leqslant (\theta d)^{\zeta+2}.
\end{equation}
\end{lemma}
\begin{proof} 
In our graphical model, all the coupling parameters are equal to $\theta$, therefore,
\begin{equation}
\J_{F\backslash H} = \I + \theta \E_{F\backslash H},
\end{equation}
where $\E_{F\backslash H}$ is the adjacency matrix of the subgraph on the vertex set $F \backslash H$. For square symmetric matrices $\M$ such that $\norm{\M} < 1$, the Maclaurin series of the function $\(\I+\M\)^{-1}$ converges uniformly on compact sets. 
Recall that the degrees of the vertices of $\mathcal{G}$ are at most $d$. This clearly applies to every its subgraph, which implies
\begin{equation}
\norm{\E_{F\backslash H}} \leqslant d,
\end{equation}
and therefore thanks to (\ref{eq:spec_assm}),
\begin{equation}
\norm{\theta\E_{F\backslash H}} \leqslant \frac{1}{2}.
\end{equation}
As a consequence, $\J_{F\backslash H}$ is invertible and we can write
\begin{equation}
\label{eq:mac_series}
\J_{F\backslash H}^{-1} = \(\I + \theta\E_{F\backslash H}\)^{-1} = \I +\sum_{i=1}^\infty (-1)^i \theta^i \E_{F\backslash H}^i,
\end{equation}
From (\ref{eq:mac_series}) we further obtain
\begin{equation}
\label{eq:mac_series1}
\J_{H, F\backslash H} \J_{F\backslash H}^{-1} \J_{F\backslash H, G\backslash F}
= \J_{H, F\backslash H}\(\I +\sum_{i=1}^\infty (-1)^i \theta^i \E_{F\backslash H}^i\)\J_{F\backslash H, G\backslash F}.
\end{equation}
Since the graph distance between $\mathcal{H}$ and $\mathcal{G} \backslash \mathcal{F}$ is at least $r$, we conclude that
\begin{equation}
\E_{H, F\backslash H} \E_{F\backslash H}^i \E_{H, F\backslash H} = \bm{0},\quad \forall \;i \leqslant r-1,
\end{equation}
therefore,
\begin{align}
\label{eq:mac_series2}
\norm{\J_{H, F\backslash H} \J_{F\backslash H}^{-1} \J_{F\backslash H, G\backslash F}} &= \theta^2\norm{\E_{H, F\backslash H}\(\sum_{i=r}^\infty (-1)^i \theta^i \E_{F\backslash H}^i\)\E_{F\backslash H, G\backslash F}} \nonumber \\
&\stackrel{(i)}{\leqslant} (d\theta)^2(d\theta)^{\zeta} = (d\theta)^{\zeta+2},
\end{align}
where in (i) we bounded the tail of Taylor's series using the remained in Lagrange's form.
\end{proof}

\begin{proof}[Proof of Lemma \ref{lem:cov_bound_inf}]
Partition the covariance $\bm\Theta$ and inverse covariance $\J$ matrices of the entire graph as
\begin{equation}
\bm\Theta = \begin{pmatrix} \bm\Theta_H & \bm\Theta_{H, F\backslash H} & \bm\Theta_{H, G\backslash F} \\ 
\bm\Theta_{H, F\backslash H}^\top & \bm\Theta_{F \backslash H} & \bm\Theta_{F\backslash H, G\backslash F} \\
\bm\Theta_{H, G\backslash F}^\top & \bm\Theta_{F\backslash H, G\backslash F}^\top & \bm\Theta_{G\backslash F} \end{pmatrix} \nonumber = \begin{pmatrix} \bm\Theta_1 & \bm\Theta_{12} & \bm\Theta_{13} \\ 
\bm\Theta_{12}^\top & \bm\Theta_2 & \bm\Theta_{23} \\
\bm\Theta_{12}^\top & \bm\Theta_{23}^\top & \bm\Theta_3 \end{pmatrix}
\end{equation}

\begin{equation}
\label{eq:inv_cov_entire_part}
\J = \begin{pmatrix} \J_H & \J_{H, F\backslash H} & \J_{H, G\backslash F} \\ 
\J_{H, F\backslash H}^\top & \J_{F \backslash H} & \J_{F\backslash H, G\backslash F} \\
\J_{H, G\backslash F}^\top & \J_{F\backslash H, G\backslash F}^\top & \J_{G\backslash F} \end{pmatrix} 
= \begin{pmatrix} \J_1 & \J_{12} & \bm{0} \\ 
\J_{12}^\top & \J_2 & \J_{23} \\
\bm{0} & \J_{23}^\top & \J_3 \end{pmatrix},
\end{equation}
where $\J_{H, G\backslash F} = \bm{0}$ in (\ref{eq:inv_cov_entire_part}) since $g\(\mathcal{H},\mathcal{G}\backslash \mathcal{F}\) > 0$. 

Our goal will now be to approximate $\J_1$ by a function of $\bm\Theta$. Using Lemma \ref{lem:shur_lemma}, let us start from the following chain of identities,
\begin{align}
\label{eq:applic_schurs}
\bm\Theta_1^{-1} &= \J_1 - \[\J_{12} \J_{13}\] 
\begin{pmatrix}
\J_2 & \J_{23} \\
\J_{23}^\top & \J_3
\end{pmatrix}^{-1} 
\begin{bmatrix} \J_{12}^\top \\ \J_{13}^\top \end{bmatrix} = \J_1 - \[\J_{12} \; \bm{0}\] 
\begin{pmatrix}
\J_2 & \J_{23} \\
\J_{23}^\top & \J_3
\end{pmatrix}^{-1} 
\begin{bmatrix} \J_{12}^\top \\ \bm{0}^\top \end{bmatrix} \nonumber \\
& = \J_1 - \J_{12} \(\J_{2} - \J_{23} \J_{3}^{-1}\J_{23}^\top \)^{-1} \J_{12}^\top \nonumber \\
& = \J_1 - \J_{12} \J_{2}^{-1/2}\(\I - \J_{2}^{-1/2}\J_{23} \J_{3}^{-1}\J_{23}^\top\J_{2}^{-1/2} \)^{-1}\J_{2}^{-1/2} \J_{12}^\top \nonumber \\
& = \J_1 - \J_{12} \J_{2}^{-1/2}\(\I + \sum_{i=1}^\infty \(\J_{2}^{-1/2}\J_{23} \J_{3}^{-1}\J_{23}^\top\J_{2}^{-1/2}\)^i \)\J_{2}^{-1/2} \J_{12}^\top \nonumber \\
& = \J_1 - \J_{12} \J_{2}^{-1} \J_{12}^\top - \J_{12} \J_{2}^{-1}\J_{23} \J_{3}^{-1}\J_{23}^\top\J_{2}^{-1} \J_{12}^\top \nonumber \\
&\qquad\qquad - \sum_{i=2}^\infty \J_{12} \J_{2}^{-1/2}\(\J_{2}^{-1/2}\J_{23} \J_{3}^{-1}\J_{23}^\top\J_{2}^{-1/2}\)^i \J_{2}^{-1/2} \J_{12}^\top.
\end{align}
Due to Lemma \ref{lem:cdp},
\begin{equation}
\norm{\J_{12} \J_{2}^{-1}\J_{23}} = \norm{\J_{H, F\backslash H} \J_{F \backslash H}^{-1}\J_{F\backslash H, G\backslash F}} \leqslant (\theta d)^{\zeta(p)+2},
\end{equation}
and the spectral norms of the other matrices in the last summand of (\ref{eq:applic_schurs}) are bounded by $2$. Therefore from (\ref{eq:applic_schurs}) we conclude,
\begin{equation}
\label{eq:applic_schurs11}
\norm{\bm\Theta_1^{-1} - \(\J_1 - \J_{12} \J_{2}^{-1} \J_{12}^\top\)} \leqslant c(\theta d)^{\zeta+2},
\end{equation}
where $c$ is a constant. In the following derivations for brevity we shall denotes by $\approx$ equalities up to a term bounded by $c(\theta d)^{\zeta+2}$. Utilizing the same technique as above, we obtain the following approximations,
\begin{align}
\begin{pmatrix} \bm\Theta_{12}^\top \\ \bm\Theta_{13}^\top \end{pmatrix} 
& = -\begin{pmatrix}
\J_2 & \J_{23} \\
\J_{23}^\top & \J_3
\end{pmatrix}^{-1} \begin{pmatrix} \J_{12}^\top \\ \J_{13}^\top \end{pmatrix}
\(\J_1 - \[\J_{12} \J_{13}\] 
\begin{pmatrix}
\J_2 & \J_{23} \\
\J_{23}^\top & \J_3
\end{pmatrix}^{-1} 
\begin{bmatrix} \J_{12}^\top \\ \J_{13}^\top \end{bmatrix}\)^{-1} \nonumber \\
& = -\begin{pmatrix}
\J_2 & \J_{23} \\
\J_{23}^\top & \J_3
\end{pmatrix}^{-1} \begin{pmatrix} \J_{12}^\top \\ \bm{0} \end{pmatrix}
\(\J_1 - \[\J_{12} \;\bm{0}\] 
\begin{pmatrix}
\J_2 & \J_{23} \\
\J_{23}^\top & \J_3
\end{pmatrix}^{-1} 
\begin{bmatrix} \J_{12}^\top \\ \bm{0}^\top \end{bmatrix}\)^{-1} \nonumber \\
& = -\begin{pmatrix}
\J_2 & \J_{23} \\
\J_{23}^\top & \J_3
\end{pmatrix}^{-1} \begin{pmatrix} \J_{12}^\top \\ \bm{0} \end{pmatrix}
\(\J_1 - \J_{12} \(\J_{2} - \J_{23} \J_{3}^{-1}\J_{23}^\top \)^{-1} \J_{12}^\top\)^{-1} \nonumber \\
& = -\begin{pmatrix}
\J_2 & \J_{23} \\
\J_{23}^\top & \J_3
\end{pmatrix}^{-1} \begin{pmatrix} \J_{12}^\top\(\J_1 - \J_{12} \(\J_{2} - \J_{23} \J_{3}^{-1}\J_{23}^\top \)^{-1} \J_{12}^\top\)^{-1} \\ \bm{0} \end{pmatrix}
 \nonumber \\
 & = -\begin{pmatrix}
\(\J_2 - \J_{23}\J_3^{-1}\J_{23}^\top\)^{-1} \\
-\J_3^{-1}\J_{23}^\top \(\J_2 - \J_{23}\J_3^{-1}\J_{23}^\top\)^{-1}
\end{pmatrix}  \J_{12}^\top\(\J_1 - \J_{12} \(\J_{2} - \J_{23} \J_{3}^{-1}\J_{23}^\top \)^{-1} \J_{12}^\top\)^{-1}
 \nonumber \\
 & \approx -\begin{pmatrix}
\J_2^{-1} \J_{12}^\top\J_1^{-1} \\
\bm{0} \end{pmatrix},
\end{align}
and
\begin{equation}
\bm\Theta_{12}\bm\Theta_2^{-1}\bm\Theta_{23} \approx \J_1^{-1} \J_{12}\J_2^{-1}\J_2\J_2^{-1} \J_{23}\J_3^{-1} = \J_1^{-1} \J_{12}\J_2^{-1} \J_{23}\J_3^{-1} \approx \bm{0}.
\end{equation}
Overall, we obtain
\begin{align}
\J_1^{-1} &= \bm\Theta_1 - \[\bm\Theta_{12} \bm\Theta_{23}\] 
\begin{pmatrix}
\bm\Theta_2 & \bm\Theta_{23} \\
\bm\Theta_{23}^\top & \bm\Theta_3
\end{pmatrix}^{-1} 
\begin{bmatrix} \bm\Theta_{12}^\top \\ \bm\Theta_{23}^\top \end{bmatrix} 
\approx \bm\Theta_1 - \[\bm\Theta_{12} \bm{0}\] 
\begin{pmatrix}
\bm\Theta_2 & \bm\Theta_{23} \\
\bm\Theta_{23}^\top & \bm\Theta_3
\end{pmatrix}^{-1} 
\begin{bmatrix} \bm\Theta_{12}^\top \\ \bm{0}^\top \end{bmatrix} \nonumber \\
&\approx \bm\Theta_1 - \bm\Theta_{12} \(\bm\Theta_2 - \bm\Theta_{23} \bm\Theta_3^{-1} \bm\Theta_{23}^\top\)^{-1} \bm\Theta_{12}^\top \approx \bm\Theta_1 - \bm\Theta_{12} \bm\Theta_2^{-1} \bm\Theta_{12}^\top,
\end{align}
which concludes the proof.
\end{proof}

\bibliographystyle{IEEEtran}
\bibliography{ilya_bib}
\end{document}